\documentclass{article}

\pdfoutput=1

\usepackage{times}
\usepackage{graphicx} 
\usepackage{subfigure} 

\usepackage[nonatbib,final]{nips_2017}

\usepackage[utf8]{inputenc} 
\usepackage[T1]{fontenc}    
\usepackage{hyperref}       
\usepackage{url}            
\usepackage{booktabs}       
\usepackage{amsfonts}       
\usepackage{nicefrac}       
\usepackage{microtype}      

\usepackage{makecell}
\usepackage{graphicx}
\usepackage{amsmath}
\usepackage{amsthm}
\usepackage{amssymb}
\usepackage{algorithm}
\usepackage{algorithmic}
\usepackage{color}

\usepackage{symbols}

\title{Dualing GANs}

\author{
  Yujia Li\thanks{Now at DeepMind.}\\
  Department of Computer Science\\
  University of Toronto\\
  \texttt{yujiali@cs.toronto.edu}\\
  \And
  Alexander Schwing\\
  Department of Electrical and Computer Engineering\\
  Coordinated Science Laboratory\\
  University of Illinois at Urbana-Champaign\\
  \texttt{aschwing@illinois.edu}
  \And
  Kuan-Chieh Wang\\
  Department of Computer Science\\
  University of Toronto\\
  \texttt{wangkua1@cs.toronto.edu}
  \And
  Richard Zemel\\
  Department of Computer Science\\
  University of Toronto\\
  \texttt{zemel@cs.toronto.edu}
}

\makeatother

\newcommand*{\ShowNotes}{}
\definecolor{darkred}{rgb}{0.7,0.1,0.1}
\definecolor{darkgreen}{rgb}{0.1,0.7,0.1}
\definecolor{cyan}{rgb}{0.7,0.0,0.7}
\definecolor{dblue}{rgb}{0.2,0.2,0.8}
\definecolor{maroon}{rgb}{0.76,.13,.28}
\definecolor{burntorange}{rgb}{0.81,.33,0}
\definecolor{tealblue}{rgb}{0.212,0.459, 0.533}

\ifdefined\ShowNotes
  \newcommand{\colornote}[3]{{\color{#1}\bf{#2: #3}\normalfont}}
\else
  \newcommand{\colornote}[3]{}
\fi

\begin{document} 

\maketitle


\begin{abstract}
Generative adversarial nets (GANs) are a promising technique for
modeling a distribution from samples.  It is however well known that  GAN training suffers from
instability due to the nature of its maximin formulation.  In this paper, we explore ways to tackle the instability problem by
dualizing the discriminator.  We start from linear discriminators in which case
conjugate duality provides a mechanism to reformulate the saddle point objective into
a maximization problem, such that both the generator and the discriminator
of this `dualing GAN' act in concert. We then demonstrate how to
extend this intuition to non-linear formulations.  For GANs with linear
discriminators our approach is able to remove the instability in training,
while for GANs with nonlinear discriminators our approach provides an
alternative to the commonly used GAN training algorithm.


\end{abstract}

\vspace{-0.2cm}
\section{Introduction}\label{sec_intro}
\vspace{-0.1cm}
Generative adversarial nets (GANs) \cite{GoodfellowARXIV2014} are, among others like variational auto-encoders
\cite{KingmaARXIV2013} and auto-regressive models \cite{vandenOordARXIV2016}, a promising technique for modeling a distribution from samples. 
A lot of empirical evidence
shows that GANs are able to learn to generate images with good visual quality at
unprecedented resolution~\cite{ZhangARXIV2016,RadfordARXIV2015}, and recently
there has been a lot of research interest in GANs, to better understand their
properties and the training process.


Training GANs can be viewed as a 
duel
between a discriminator and a generator. Both players are instantiated as deep
 nets.
The generator is required to produce realistic-looking
samples that cannot be differentiated from real data by the discriminator. 
In turn, the discriminator does as good a job as possible to tell the samples apart from
real data.
Due to the complexity of the optimization problem, training GANs is 
notoriously hard, and usually suffers from problems such as mode collapse,
vanishing gradient, and divergence. The training procedures are very unstable and
sensitive to hyper-parameters.  A number of techniques have been proposed to
address these issues, some empirically justified~\cite{RadfordARXIV2015,
SalimansARXIV2016}, and some more theoretically motivated~\cite{MetzARXIV2016,
ArjovskyARXIV2017,NowozinNIPS2016, ZhaoICLR2017}.

This tremendous amount of recent work,
together with the wide variety of heuristics applied by practitioners,
indicates that many questions regarding the properties of GANs are
still unanswered. In this work we provide another perspective on the
properties of GANs, aiming toward better training algorithms in some cases.
Our study in this paper is motivated by the alternating gradient update
between discriminator and generator, employed during training of GANs.  This
form of update is one source of instability, and it is known to diverge even for
some simple problems~\cite{SalimansARXIV2016}. Ideally, when the discriminator
is optimized to optimality, the GAN objective is a deterministic function of the
generator. In this case, the optimization problem would be much easier to solve.
This motivates our idea to dualize parts of the GAN objective, offering a
mechanism to better optimize the discriminator.

Interestingly, our dual formulation provides a direct relationship between the
GAN objective and the maximum mean-discrepancy framework discussed
in \cite{GrettonJMLR2012}. When restricted to linear discriminators, where
we can find the optimal discriminator by solving the dual, this formulation
permits the derivation of an optimization algorithm that monotonically increases the
objective. 
Moreover, for non-linear discriminators we
can apply trust-region type optimization techniques to obtain more accurate
discriminators.  Our work brings to the table some additional
optimization techniques beyond stochastic gradient descent; we
hope this encourages other researchers to pursue this direction.




\vspace{-0.2cm}
\section{Background}
\vspace{-0.1cm}


In generative training we are interested in modeling of and sampling from an unknown distribution
$P$, given a set $\cD = \{\bx_1, \ldots, \bx_N\} \sim P$ of datapoints, for
example images.
GANs use a \emph{generator} network $G_\theta(\bz)$ parameterized by $\theta$,
that maps samples $\bz$ drawn from a simple distribution, \eg, Gaussian or uniform, to samples in the data space
$\hat{\bx} = G_\theta(\bz)$.  A separate
\emph{discriminator} $D_\bw(\bx)$ parameterized by $\bw$ maps a point $\bx$ in the
data space to the probability of it being a real sample.

The discriminator is trained to minimize a classification loss, typically the
cross-entropy, and the generator is trained to maximize the same loss.  On sets
of real data samples $\{\bx_1, ..., \bx_n\}$ and noise samples $\{\bz_1, ...,
\bz_n\}$, using the 
(averaged) cross-entropy loss results in the following joint optimization problem: 
\be
  \label{eq:GanOrig}
  \max_\theta \min_\bw f(\theta, \bw) \quad
  \text{where} \quad f(\theta, \bw) = -\frac{1}{2n}\sum_i \log D_\bw(\bx_i) -
  \frac{1}{2n}\sum_i \log
  (1 - D_\bw(G_\theta(\bz_i))).
\ee
We adhere to the
formulation of a fixed batch of samples for clarity of the presentation, but also point out how this process is
adapted to the stochastic optimization setting later in the paper as well as 
in the supplementary material.

To solve this maximin optimization problem, ideally, 
we want to solve for the optimal discriminator parameters $\bw^*(\theta) = \argmin_\bw
f(\theta, \bw)$, in which case the GAN program given in \equref{eq:GanOrig} can be
reformulated as a  maximization for $\theta$ using $\max_\theta f(\theta,
\bw^*(\theta))$.  However,
typical GAN training only
alternates two gradient updates $\bw\leftarrow
\bw - \nabla_\bw f(\theta, \bw)$ and $\theta\leftarrow \theta + \nabla_\theta f(\theta,
\bw)$, and usually just one step for each of $\theta$ and $\bw$ in each round.
In this case, the objective to be maximized by the generator
is $f(\theta, \bw)$ instead. This objective is always an upper bound on the
correct objective $f(\theta, \bw^*(\theta))$, since $\bw^*(\theta)$ is the
optimal $\bw$ 
for $\theta$.  Maximizing an upper bound has no guarantee on maximizing the
correct objective, which leads to instability. Therefore, many practically
useful techniques have been proposed to circumvent the difficulties of the
original program definition presented in \equref{eq:GanOrig}.

Another widely employed technique is a separate loss $-\sum_i \log
(D_\bw(G_\theta(\bz_i)))$ to update $\theta$ in order to avoid vanishing
gradients during early stages of training when the discriminator can get too
strong.  This technique can be combined with our approach, but in what follows,
we keep the elegant formulation of the GAN program specified in
\equref{eq:GanOrig}. 


\vspace{-0.2cm}
\section{Dualing GANs}
\vspace{-0.1cm}

The main idea of `Dualing GANs' is to represent the discriminator program
$\min_\bw f(\theta, \bw)$ in \equref{eq:GanOrig} using its dual, $\max_\lambda
g(\theta, \lambda)$. Hereby, $g$ is the dual objective of $f$ \wrt $\bw$, and
$\lambda$ are the dual variables.  Instead of gradient descent on $f$ to update
$\bw$,  we solve the dual instead.  This results in a 
maximization problem $\max_\theta \max_\lambda g(\theta,
\lambda)$. 

Using the dual is beneficial for two reasons. First, note that for any
$\lambda$, $g(\theta, \lambda)$ is a lower bound on the objective with optimal
discriminator parameters $f(\theta, \bw^*(\theta))$. Staying in the dual domain,
we are then guaranteed that optimization of $g$ \wrt $\theta$ makes progress in
terms of the original program.  Second, the dual problem usually involves a much
smaller number of variables, and can therefore be solved much more easily than
the primal formulation. This provides opportunities to obtain more accurate
estimates for the discriminator parameters $\bw$, which is in turn beneficial
for stabilizing the learning of the generator parameters $\theta$.  In the following, we start by
studying linear discriminators, before extending our technique to training with
non-linear discriminators. Also, we use 
cross-entropy as the classification loss,
but emphasize that other convex loss functions, \eg, the hinge-loss, can be
applied equivalently. 

\vspace{-0.1cm}
\subsection{Linear Discriminator}
\label{sec:LD}
\vspace{-0.1cm}

We start from linear discriminators that use a linear scoring function $F(\bw,
\bx)=\bw^\top\bx$, \ie, the discriminator $D_\bw(\bx) = p_\bw(y=1|\bx) =
\sigma(F(\bw, \bx))=1/[1 + \exp(-\bw^\top\bx)]$.  Here, $y=1$ indicates real
data,
while $y=-1$ for a generated sample, and 
$p_\bw(y=-1|\bx)  = 1 - p_\bw(y=1|\bx)$ characterizes the probability of $\bx$
being a generated versus real data sample.

We only require the scoring function $F$ to be linear in $\bw$ and any
(nonlinear) differentiable features $\phi(\bx)$ can be used in place of $\bx$ in this
formulation.  Substituting the linear scoring function into the objective given in \equref{eq:GanOrig}, 
results in the following program for $\bw$:
\be
\min_\bw \quad \frac{C}{2} \|\bw\|^2_2 + \frac{1}{2n}\sum_i \log(1 +
\exp(-\bw^\top\bx_i)) + \frac{1}{2n}\sum_i \log (1 + \exp(\bw^\top
G_\theta(\bz_i))).
\label{eq:LogLossPrimal}
\ee
Here we also added an L2-norm regularizer on $\bw$.
We note that the
program presented in \equref{eq:LogLossPrimal} is convex in the discriminator
parameters $\bw$. Hence, we can equivalently solve it in the dual domain as
discussed in the following claim, with proof provided in the supplementary
material.

\begin{claim}
\label{clm:dualLogLoss}
The dual program to the task given in \equref{eq:LogLossPrimal} reads as follows:
\begin{align}
  \max_\lambda \quad& g(\theta, \lambda) = -\frac{1}{2C}\left\|\sum_i \lambda_{\bx_i} \bx_i - \sum_i \lambda_{\bz_i}
  G_\theta(\bz_i)\right\|^2 + \frac{1}{2n}\sum_i H(2n\lambda_{\bx_i}) +
  \frac{1}{2n}\sum_i H(2n\lambda_{\bz_i}), \nonumber\\
  \suchthat \quad& \forall i, \quad 0 \le \lambda_{\bx_i} \le \frac{1}{2n}, \quad 0 \le
  \lambda_{\bz_i} \le \frac{1}{2n},
\end{align}
with binary entropy $H(u) = - u\log u-(1-u)\log(1-u)$. The optimal solution
to the original problem $\bw^*$ can be obtained from the optimal
$\lambda_{\bx_i}^*$ and $\lambda_{\bz_i}^*$ via
$$
\bw^* = \frac{1}{C}\left(\sum_i \lambda_{\bx_i}^* \bx_i - \sum_i
  \lambda_{\bz_i}^* G_\theta(\bz_i)\right).
$$

\end{claim}
\textbf{Remarks: } Intuitively, considering the last two terms of the program
given in \clmref{clm:dualLogLoss} as well as its constraints, we aim at
assigning weights $\lambdax$, $\lambdaz$ close to half of $\frac{1}{2n}$ to as many data points
and to as many artificial samples as possible. More carefully investigating the
first part, which can at most reach zero, reveals that we aim to match the
empirical data observation $\sum_i \lambda_{\bx_i}\bx_i$ and the generated
artificial sample observation $\sum_i \lambda_{\bz_i} G_\theta(\bz_i)$. Note that this
resembles the moment matching property obtained in other maximum likelihood
models. Importantly, this objective also resembles the (kernel) maximum mean
discrepancy (MMD) framework, where the empirical squared MMD is estimated via
$\|\frac{1}{n}\sum_{\bx_i} \bx_i - \frac{1}{n}\sum_{\bz_i}
G_\theta(\bz_i)\|_2^2$. 
Generative
models that learn to minimize the MMD objective, like the generative moment
matching networks~\cite{LiARXIV2015,dziugaite2015training}, can
therefore be included in our framework,
using fixed $\lambda$'s and proper scaling of the first term. 

Combining the result obtained in \clmref{clm:dualLogLoss}  with the training
objective for the generator yields the task $\max_{\theta,\lambda} g(\theta, \lambda)$  for training of GANs with linear
discriminators. 
Hence, instead of searching for a saddle-point, we strive to find a maximizer, a
task which is presumably easier. The price to pay is the restriction to linear
discriminators and the fact that every randomly drawn artificial sample $\bz_i$
has its own dual variable $\lambda_{\bz_i}$.

In the non-stochastic optimization setting, where we optimize for fixed sets of
data samples $\{\bx_i\}$ and randomizations $\{\bz_i\}$, it is easy to design a learning
algorithm for GANs with linear discriminators that monotonically improves the
objective $g(\theta, \lambda)$ based on line search. Although
this approach is not practical for very large data sets, such a property is
convenient for smaller scale data sets.  In addition, linear models
are favorable in scenarios in which we know 
informative features that we want the discriminator to pay attention to.

When optimizing with mini-batches we introduce new data samples $\{\bx_i\}$ and
randomizations $\{\bz_i\}$ in every iteration.  In the supplementary material we
show that this corresponds to maximizing a lower bound on the full expectation
objective.  Since the dual variables vary from one
mini-batch to the next, we need
to solve for the newly introduced dual variables to a reasonable accuracy.
For small minibatch sizes commonly used in deep learning literature, like 100,
calling a constrained optimization solver to solve the dual problem is quite
cheap. We used Ipopt~\cite{WaechterMP2006}, which  solves this dual
problem to a very good accuracy in negligible time; other solvers can
also be used and may lead to improved performance.

Utilizing a log-linear discriminator reduces the model's expressiveness and complexity.
We therefore now propose methods to alleviate this restriction.

\begin{figure}[t]
\centering
\fbox{
\begin{minipage}{0.95\linewidth}
Initialize $\theta$, $\bw_0$, $k=0$ and iterate
\begin{enumerate}\setlength\itemsep{0pt}\setlength{\parskip}{0pt}\setlength{\parsep}{0pt}
\item One or few gradient ascent steps on $f(\theta,\bw_k)$ \wrt generator parameters $\theta$
\item Find step $\bs$ using $\min_\bs m_{k,\theta}(\bs)$ s.t.~$\frac{1}{2}\|\bs\|_2^2\leq\Delta_k$
\item Update $\bw_{k+1} \leftarrow \bw_k + \bs$
\item $k \leftarrow k + 1$
\end{enumerate}
\end{minipage}
}
\caption{GAN optimization with model function.}
\label{fig:AlgoOutline}
\end{figure}



\subsection{Non-linear Discriminator}
\label{sec:NLD}
General non-linear discriminators use non-convex scoring functions $F(\bw, \bx)$, parameterized by a deep  net.  
The non-convexity of $F$ makes it hard to directly convert the problem into its dual form. 
Therefore, our approach for training GANs with non-convex discriminators is based on
repeatedly linearizing and dualizing the discriminator locally. At first sight
this seems restrictive, however, we will show that a specific setup of this
technique recovers the gradient direction employed in the regular GAN training
mechanism while providing additional flexibility.

We consider locally approximating the primal objective $f$ around a point $\bw_k$
using a model function $m_{k,\theta}(\bs)\approx f(\theta, \bw_k+\bs)$. We phrase the update
\wrt the discriminator parameters $\bw$ as a search for a step $\bs$, \ie,
$\bw_{k+1} = \bw_{k} + \bs$ where $k$ indicates the current iteration.  In order
to guarantee the quality of the approximation, we introduce a trust-region
constraint $\frac{1}{2}\|\bs\|_2^2 \leq \Delta_k \in \mathbb{R}^+$ where
$\Delta_k$ specifies the trust-region size.  More concretely, 
we search for a step $\bs$ by solving
\be
\min_\bs m_{k,\theta}(\bs) \quad\text{s.t.}\quad \frac{1}{2}\|\bs\|_2^2\leq \Delta_k,
\label{eq:ModelPrimal}
\ee
given generator parameters $\theta$. Rather than optimizing the GAN
objective $f(\theta,\bw)$ 
with stochastic gradient descent, we can instead employ this model function and use the
algorithm outlined in \figref{fig:AlgoOutline}. It proceeds by first performing
a gradient ascent \wrt the generator parameters $\theta$. Afterwards, we find a
step $\bs$ by solving the program given in \equref{eq:ModelPrimal}. We then
apply this step, and repeat.

Different model functions $m_{k,\theta}(\bs)$ result in  variants of the
algorithm.  If we choose $m_{k,\theta}(\bs) = f(\theta,\bw_k+\bs)$, model $m$
and function $f$ are identical but the program given in \equref{eq:ModelPrimal}
is hard to solve. Therefore, in the following, we propose two  model functions
that we have found to be useful. The first one is based on linearization of the cost
function $f(\theta,\bw)$ and recovers the step $\bs$ employed by gradient-based
discriminator updates in standard GAN training. The second one is based on
linearization of the score function $F(\bw,\bx)$ while keeping the loss function
intact; this second approximation is hence accurate in a larger region. Many
more models $m_{k,\theta}(\bs)$ exist and we leave further exploration of this
space to future work.

{\bf (A). Cost function linearization:}
A local approximation to the cost function $f(\theta,\bw)$ 
can be constructed by using the first order Taylor approximation
$$
m_{k,\theta}(\bs) = f(\bw_k,\theta) + \nabla_\bw f(\bw_k,\theta)^\top\bs.
$$
Such a model function is appealing because step 2 of the algorithm outlined in
\figref{fig:AlgoOutline}, \ie, minimization of the model function subject to
trust-region constraints as specified in \equref{eq:ModelPrimal}, has the
analytically computable solution
$$
\bs = -\frac{\sqrt{2\Delta_k}}{\|\nabla_\bw f(\bw_k,\theta)\|_2} \nabla_\bw f(\bw_k,\theta).
$$
Consequently step 3 of the algorithm outlined in \figref{fig:AlgoOutline} is a
step of length $2\Delta_k$ into the negative gradient direction of the cost
function $f(\theta,\bw)$. We can use the trust region parameter $\Delta_k$ to
tune the step size just like it is common to specify the step size for standard
GAN training. As mentioned before, using the first order Taylor approximation as
our model $m_{k,\theta}(\bs)$ recovers the same direction that is employed
during standard GAN training.  The value of the $\Delta_k$ parameters can be
fixed or adapted; see the supplementary material for more details.

Importantly, using the first order Taylor approximation as a model is not the
only choice. While some choices are fairly obvious, such as a quadratic
approximation, we present another intriguing option in the following.

{\bf (B). Score function linearization:}
Instead of linearizing the entire cost function as demonstrated in the previous
part, we can choose to only linearize the score function $F$, locally around
$\bw_k$, via
$$
F(\bw_k + \bs,\bx) \approx \hat F(\bs,\bx) = F(\bw_k,\bx) + \bs^\top\frac{\partial
F(\bw_k,\bx)}{\partial \bw}, \quad \forall \bx.
$$
Note that the overall objective $f$ is itself a nonlinear function of $F$. 
Substituting the approximation for $F$ into the overall objective, results in 
the following model function:
\begin{align}
  m_{k, \theta}(\bs) =\,& \frac{C}{2}\|\bw_k + \bs\|_2^2 + \frac{1}{2n}\sum_i
  \log\left(1 + \exp\left(-F(\bw_k, \bx_i) - \bs^\top\pdiff{F(\bw_k,
  \bx_i)}{\bw}\right)\right) \nonumber\\
  \label{eq:ModelFunction}
  & + \frac{1}{2n}\sum_i \log \left(1 + \exp\left(F(\bw_k, G_\theta(\bz_i)) +
  \bs^\top\pdiff{F(\bw_k, G_\theta(\bz_i))}{\bw}\right)\right).
\end{align}
This approximation keeps the nonlinearities of the surrogate loss function intact,
therefore we expect it to be more accurate than
linearization of the whole cost function $f(\theta,\bw)$. When $F$ is already
linear in $\bw$, 
linearization of the
score function introduces no approximation error, and the formulation can be
naturally reduced to the discussion presented in \secref{sec:LD};
non-negligible errors are introduced when linearizing the whole cost function $f$ in
this case.

For general non-linear discriminators, however, no analytic solution can be
computed for the program given in \equref{eq:ModelPrimal} when using this model.
Nonetheless, the model function fulfills $m_{k,\theta}(0) = f(\bw_k,\theta)$
and it is convex in $\bs$. Exploiting this convexity, we can derive the dual
for this trust-region optimization problem as presented in the following claim.
The proof is included in the supplementary material.
\begin{claim}
\label{clm:dualModel}
The dual program to $\min_\bs m_{k,\theta}(\bs)$ s.t.~$\frac{1}{2}\|\bs\|_2^2
\leq \Delta_k$ with model function as in \equref{eq:ModelFunction} is:
\begin{eqnarray*}
  \max_{\lambda}
  &&\frac{C}{2}\|\bw_k\|_2^2 -
  \frac{1}{2(C+\lambda_T)}\left\|-C\bw_k + \sum_i \lambda_{\bx_i}\pdiff{F(\bw_k,
  \bx_i)}{\bw} - \sum_i \lambda_{\bz_i}\pdiff{F(\bw_k,
  G_\theta(\bz_i))}{\bw}\right\|_2^2 \\
  &&+ \frac{1}{2n}\sum_i H(2n\lambda_{\bx_i}) + \frac{1}{2n}\sum_i
H(2n\lambda_{\bz_i}) - \sum_i \lambda_{\bx_i} F_{\bx_i} + \sum_i \lambda_{\bz_i}
F_{\bz_i} -\lambda_T \Delta_k \\
\suchthat && \lambda_T \geq 0 \quad\quad \forall i,\quad 0\leq\lambda_{\bx_i}\leq
\frac{1}{2n}, \quad 0\leq \lambda_{\bz_i}\leq \frac{1}{2n}.
\end{eqnarray*}
The optimal $\bs^*$ to the original problem can
be expressed through optimal $\lambda_T^*, \lambda_{\bx_i}^*, \lambda_{\bz_i}^*$ as
\begin{equation*}
  \bs^*= \frac{1}{C + \lambda_T^*}\left(\sum_i \lambda_{\bx_i}^*\pdiff{F(\bw_k,
      \bx_i)}{\bw}
    - \sum_i \lambda_{\bz_i}^*\pdiff{F(\bw_k, \bz_i)}{\bw}\right) - \frac{C}{C +
    \lambda_T^*}\bw_k
\end{equation*}
\end{claim}

Combining the dual formulation with the maximization of the
generator parameters $\theta$ results in a maximization as opposed to a
search for a saddle point.
However, unlike the linear case, it is not possible to design an
algorithm that is guaranteed to monotonically increase the cost function
$f(\theta,\bw)$. The culprit is step 3 of the algorithm outlined in
\figref{fig:AlgoOutline}, which adapts the model $m_{k,\theta}(\bs)$ in every
iteration.
 

Intuitively, the program illustrated in \clmref{clm:dualModel} aims at choosing
dual variables $\lambda_{\bx_i}$, $\lambda_{\bz_i}$ such that the weighted means of
derivatives as well as scores match. Note that this program searches for
a direction $\bs$ as opposed to searching for the weights $\bw$, hence the term
$-C\bw_k$ inside the squared norm.

In practice, we use Ipopt~\cite{WaechterMP2006} to solve the dual problem.
The form of this dual is  more ill-conditioned than the linear case. 
The solution found by Ipopt sometimes contains errors, however, we found the
errors to be generally tolerable and not to affect the performance of our models.

\section{Experiments}\label{sec_exp}

\begin{figure}[t]
\centering
\begin{tabular}{cc}
  \includegraphics[width=0.45\textwidth]{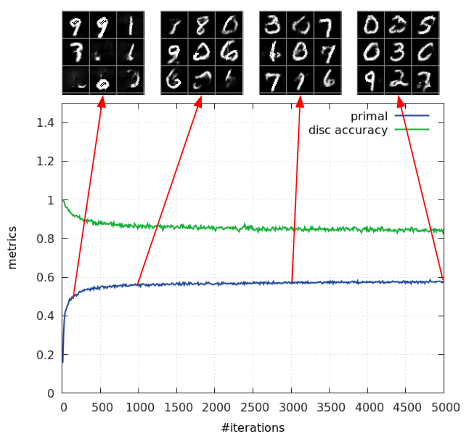} &
\includegraphics[width=0.45\textwidth]{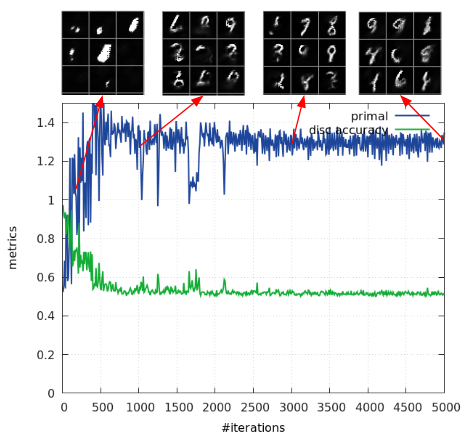} \\
\includegraphics[width=0.25\textwidth]{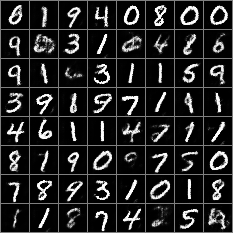} &
\includegraphics[width=0.25\textwidth]{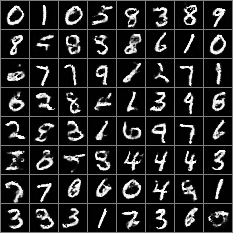}
\end{tabular}
\caption{
We show the learning curves and samples from two models of the same
architecture, one optimized in dual space (left), and one in the primal space (\ie, typical
GAN) up to 5000 iterations.  Samples are shown at different points during
training, as well as at the very end (second row).
Despite having similar sample qualities in the end, they demonstrate drastically
different training behavior.  In the typical GAN setup, loss oscillates and has 
no clear trend, whereas in the dual setup, loss monotonically increases and shows
much smaller oscillation. Sample quality is nicely correlated with the dual objective
during training.}
\label{fig:lin-mnist}
\end{figure}

\begin{figure}[t]
\centering
\begin{tabular}{c}
\includegraphics[width=0.95\textwidth]{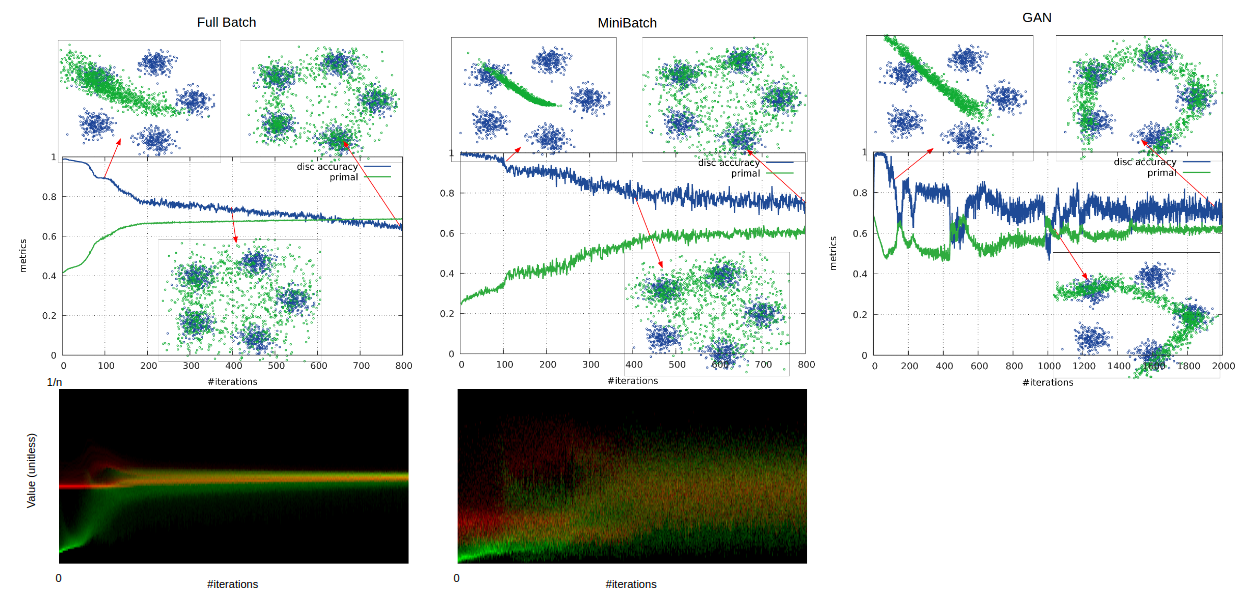}
\end{tabular}
\caption{
  Training GANs with linear discriminators on the simple 5-Gaussians dataset.
  Here we are showing typical runs with the compared methods (not
  cherry-picked).
Top: training curves and samples
from a single experiment: left - dual with full batch, middle - dual
with minibatch, right - standard GAN with minibatch.
  The real data from this dataset
are drawn in blue, generated samples in green.
Below: distribution of
$\lambda$'s during training for the two dual GAN experiments, as a histogram
at each x-value (iteration) where intensity depicts frequency for values
ranging from 0 to 1 (red are data, and green are samples).
}
\label{fig:lin-toy}
\end{figure}

In this section, we empirically study the proposed dual GAN algorithms.  
In particular, we show the stable and monotonic training for linear
discriminators and study its properties. For nonlinear GANs we show good quality
samples and compare it with standard GAN training methods.  Experiments are done
on three datasets: a  2D dataset composed of 5 2D Gaussians (5-Gaussians), MNIST
\cite{LeCunIEEE1998}, and CIFAR-10 \cite{krizhevsky2009learning}.
Overall the results
show that our proposed approaches work across a range of problems and provide
good alternatives to the standard GAN training method.

\subsection{Dual GAN with linear discriminator}
\label{sec:expLinGan}

We explore the dual GAN with linear discriminator on the synthetic 2D dataset
generated by sampling points from a mixture of 5 2D Gaussians,
as well as the MNIST dataset.
Through these experiments we show that (1) with the proposed dual GAN algorithm,
training is very stable; (2) the dual variables $\lambda$ can be used as an extra
informative signal for monitoring the training process; (3) features matter,
and we can train good generative models even with linear discriminators when we
have good features.  In all experiments, we compare our
proposed dual GAN with standard GAN, for training the same generator and
discriminator models.

The discussion of linear discriminators presented in \secref{sec:LD} works with any
feature representation $\phi(\bx)$ in place of $\bx$ as long as $\phi$ is differentiable to allow
gradients flow through it.
For the simple 5-Gaussian dataset, we use RBF features based on 100 sample training points.
For the MNIST dataset, we use a convolutional
neural net, and concatenate the hidden activations on all layers as the features. 

The dual GAN formulation has a single hyper-parameter $C$, but we found the
algorithm not to be sensitive to it, and set it to 0.0001 in all experiments.
We used Adam \cite{KingmaICLR2015} with fixed
learning rate and momentum to optimize the generator.  Additional experimental
details and results are included in the supplementary material.


\textbf{Stable Training:} The main results illustrating stable training are provided in
\figref{fig:lin-mnist} and \ref{fig:lin-toy}, where we show the learning
curves as well as model samples at different points during training.  Both the
dual GAN and the standard GAN use minibatches of the same size, and for the synthetic 
dataset we did an extra experiment doing full-batch training.  From these curves
we can see the stable monotonic increase of the dual objective, contrasted with
standard GAN's spiky training curves.  On the synthetic data, we see that
increasing the minibatch size leads to significantly improved stability.
In the
supplementary material we include an extra experiment to quantify the stability
of the proposed method on the synthetic dataset.

\textbf{Sensitivity to Hyperparameters:} Sensitivity to hyperparameters is
another important aspect of training stability. Successful GAN training typically requires
carefully tuned hyperparameters, making it difficult for non-experts to adopt
these generative models.  In an attempt to quantify this
sensitivity, we investigated the robustness of the proposed method to
hyperparameter choice,
and empirically showed the
proposed method was less sensitive to the choice of hyperparameters. For both the
5-Gaussians and MNIST datasets, we randomly sampled 100 hyperparameter settings
from ranges specified in Table \ref{tab:hyperp-ranges}, and compared learning using both
the proposed dual GAN and the standard GAN.  On the 5-Gaussians dataset, we
evaluated the performance of the models by how well
the model samples covered the 5 modes.  We defined successfully covering a mode
as having $>100$ out of $1000$ samples falling within a distance of 3 standard
deviations to the center of the Gaussian. Our dual linear GAN
succeeded in 49\% of the experiments, and standard GAN succeeded in only 32\%,
demonstrating our method was significantly easier to train and tune.  On MNIST,
the mean Inception scores were 2.83, 1.99 for the proposed method and GAN
training respectively. A more detailed breakdown of mode coverage and Inception
score can be found in Figure \ref{fig:hyperp}.

\begin{table}[t]
\setlength\tabcolsep{1.5pt} 
\footnotesize
  \centering
  {\footnotesize
  \begin{tabular}{r|c|c|c|c|c|c|c}
    \toprule
    Dataset & mini-batch size & generator & generator & $C$ & discriminator & generator & max \\
     &  & learnrate & momentum & & learnrate* & architecture& iterations \\
    \hline
    5-Gaussians& randint[20,200] & enr([0,10])&rand[.1,.9] & enr([0,6]) & enr([0,10]) & fc-small& randint[400,2000] \\
    &&&&&& fc-large&\\
    \hline
    MNIST& randint[20,200] & enr([0,10])&rand[.1,.9] & enr([0,6]) & enr([0,10]) & fc-small& 20000 \\
    &&&&&& fc-large&\\
    &&&&&& dcgan&\\
    &&&&&& dcgan-no-bn&\\
    \bottomrule
  \end{tabular}}
  \caption{Ranges of hyperparameters for sensitivity experiment. randint[a,b]
  means samples were drawn from uniformly distributed integers in the closed
interval of [a,b], similarly rand[a,b] for real numbers. enr([a,b]) is
shorthand for exp(-randint[a,b]), which was used for hyperparameters commonly
explored in log-scale. For generator architectures, for the 5-Gaussians dataset
we tried 2 3-layer fully-connected networks, with 20 and 40 hidden units.  For
MNIST, we tried 2 3-layer fully-connected networks, with 256 and 1024 hidden
units, and a DCGAN-like architecture with and without batch normalization.}
  \label{tab:hyperp-ranges}
\end{table}

\begin{figure*}[t]
  \centering
  \begin{tabular}{cc}
    \includegraphics[width=0.4\textwidth]{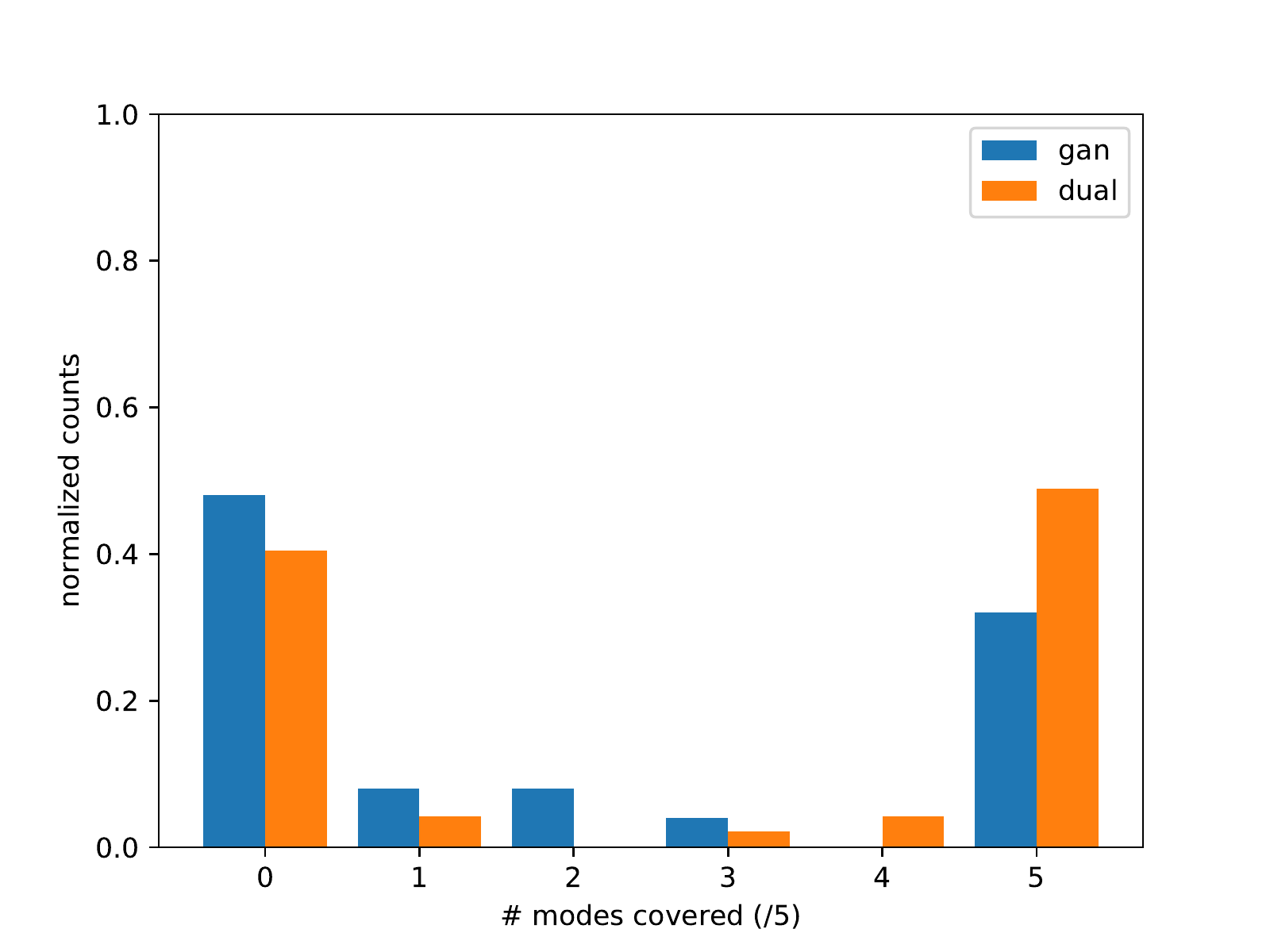} &
    \includegraphics[width=0.4\textwidth]{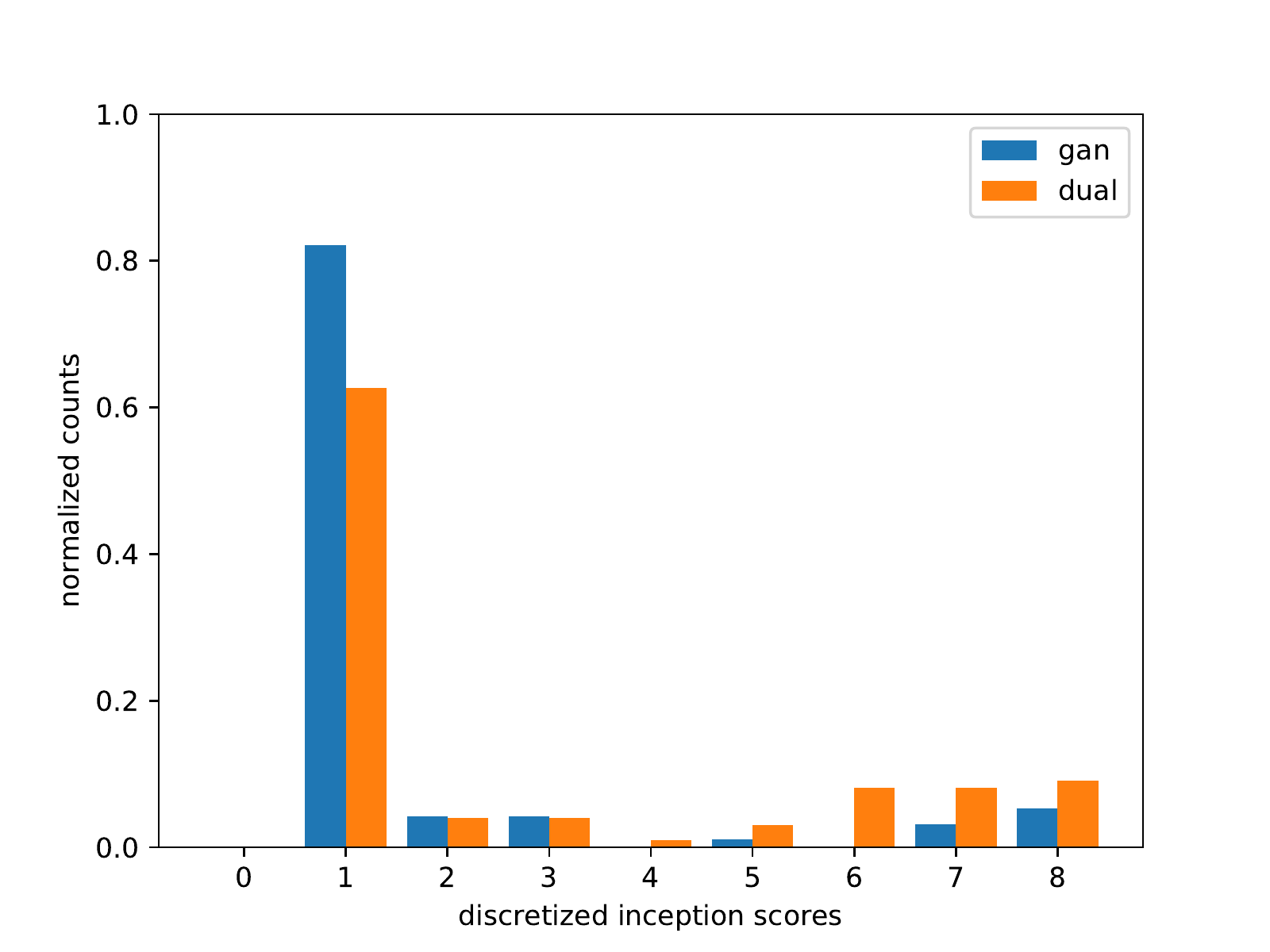}\\
    5-Gaussians & MNIST \\
  \end{tabular}
  \caption{Results for hyperparameter sensitivity experiment. For 5-Gaussians
  dataset, the x-axis represents the number of modes covered. For MNIST, the
x-axis represents discretized Inception score. Overall, the proposed dual GAN
results concentrate significantly more mass on the right side, demonstrating its
better robustness to hyperparameters than standard GANs.}
  \label{fig:hyperp}
\end{figure*}



\textbf{Distribution of $\lambda$ During Training:} The dual formulation
allows us to monitor the training process through a unique perspective by
monitoring the dual variables $\lambda$.   \figref{fig:lin-toy} shows the evolution of the
distribution of $\lambda$ during training for the synthetic 2D dataset.  At the
begining of training the $\lambda$'s are on the low side as the generator is not
good and $\lambda$'s are encouraged to be small to minimize the moment matching
cost.  As the generator improves, more attention is devoted to the entropy term in the
dual objective, and the $\lambda$'s start to converge to the value of $1/4n$.

\textbf{Comparison of Different Features:} The qualitative differences of the learned models
with different features can be observed in  \figref{fig:lin-mnist-feature}.  In general, the
more information the features carry about the data, the better the learned
generative models are. 
On MNIST, even with random features and linear discriminators we
can learn reasonably good generative models.
On the other hand, these results also indicate that if the features are bad
then it is hard to learn good models.
This leads us to the nonlinear
discriminators presented below, where the discriminator features are learned
together with the last layer, which may be necessary for more complicated problems
domains where features are potentially difficult to engineer.



\begin{figure*}[t]
  \centering
  \begin{tabular}{ccccccc}
    \rotatebox{90}{Trained}&
    \includegraphics[trim={8cm 0 0 0},clip,width=0.13\textwidth]{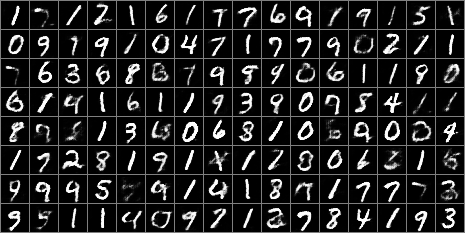} &
    \includegraphics[trim={8cm 0 0 0},clip,width=0.13\textwidth]{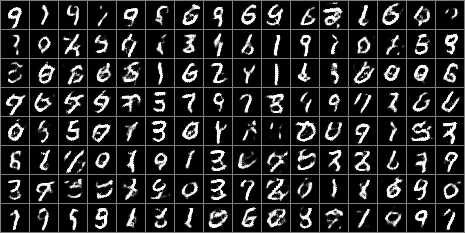} &
    \includegraphics[trim={8cm 0 0 0},clip,width=0.13\textwidth]{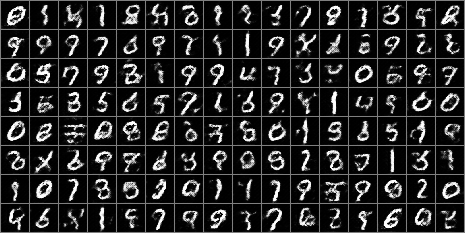} &
    \includegraphics[trim={8cm 0 0 0},clip,width=0.13\textwidth]{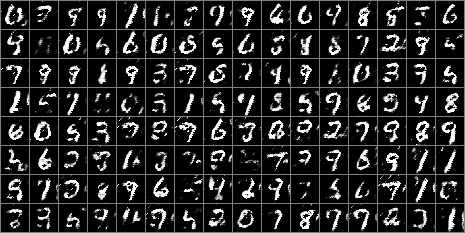} &
    \includegraphics[trim={8cm 0 0 0},clip,width=0.13\textwidth]{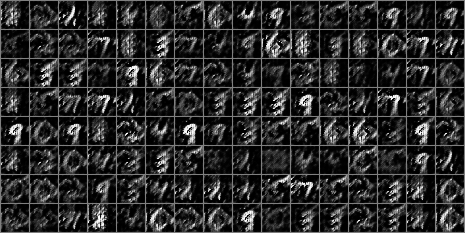}&
    \includegraphics[trim={8cm 0 0 0},clip,width=0.13\textwidth]{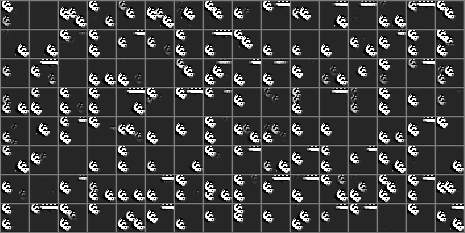}\\

    \rotatebox{90}{Random}&
    \includegraphics[trim={8cm 0 0 0},clip,width=0.13\textwidth]{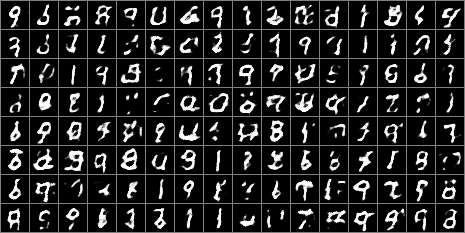} &
    \includegraphics[trim={8cm 0 0 0},clip,width=0.13\textwidth]{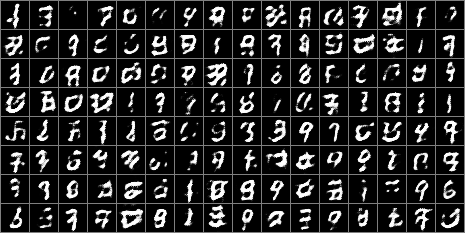} &
    \includegraphics[trim={8cm 0 0 0},clip,width=0.13\textwidth]{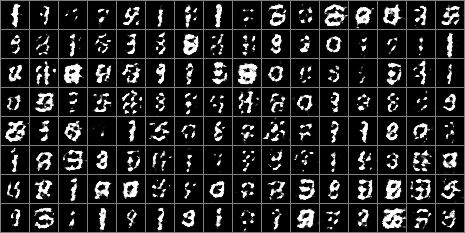} &
    \includegraphics[trim={8cm 0 0 0},clip,width=0.13\textwidth]{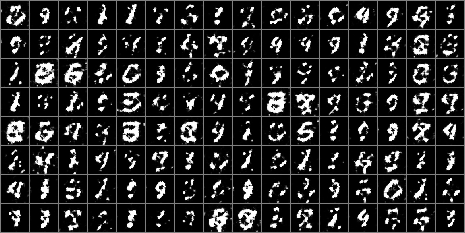} &
    \includegraphics[trim={8cm 0 0 0},clip,width=0.13\textwidth]{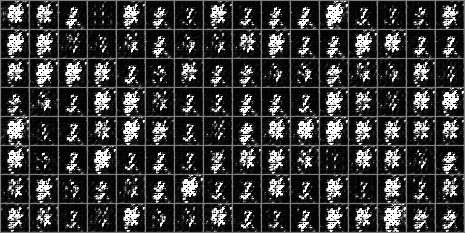}&
    \includegraphics[trim={8cm 0 0 0},clip,width=0.13\textwidth]{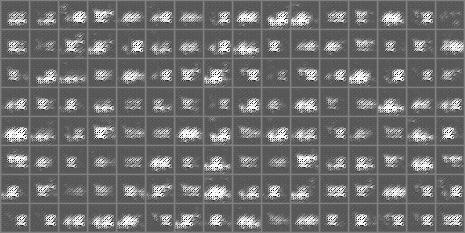}\\
    & Layer: All & Conv1 & Conv2 &Conv3 & Fc4 & Fc5\\
  \end{tabular}
  \caption{Samples from dual linear GAN using pretrained and random features
  on MNIST. Each column shows a different features, utilizing all layers in
  a convnet and then successive single layers in the network.}
  \label{fig:lin-mnist-feature}
\end{figure*}

\subsection{Dual GAN with non-linear discriminator}
\label{sec:expNonLinGan}
Next we assess the applicability of our proposed technique for non-linear
discriminators, and focus on training models on 
MNIST and CIFAR-10.

As discussed in \secref{sec:NLD}, when the discriminator is non-linear, we can
only approximate the discriminator locally. 
Therefore we do not have monotonic convergence guarantees. 
However, through better approximation and optimization of the
discriminator we may expect the proposed dual GAN to work better than standard
gradient based GAN training in some cases.
Since GAN training is sensitive to hyperparameters, to make the comparison fair,
we tuned the parameters for both the standard GANs and our approaches
extensively and compare the best results for each.

\figref{fig:nonlin-mnist} and \ref{fig:nonlin-cifar} show the samples 
generated by models learned using different
approaches.  Visually samples of our proposed approaches are on par with the
standard GANs.  As an extra quantitative metric for performance, we computed the
Inception Score~\cite{SalimansARXIV2016} for each of them on CIFAR-10 in Table~\ref{tab:inception}. 
The Inception Score is a surrogate metric which highly depends on the network architecture. Therefore we computed the score  using our own classifier and the one proposed in~\cite{SalimansARXIV2016}. As can be seen in Table~\ref{tab:inception}, both score and cost
linearization are competitive with standard GANs.
%
From the training curves we can also see that
score linearization does the best in terms of approximating the objective, and
both score linearization and cost linearization oscillate less than standard
GANs.

\begin{table}[t]
  \centering
  {
  \begin{tabular}{r|cccc}
    \toprule
    Score Type & GAN & Score Lin & Cost Lin & Real Data \\
    \hline
    Inception (end)           & 5.61$\pm$0.09   & 5.40$\pm$0.12   & 5.43$\pm$0.10  & 10.72 $\pm$ 0.38 \\
    Internal classifier (end) & 3.85$\pm$0.08   & 3.52$\pm$0.09   & 4.42$\pm$0.09  & 8.03  $\pm$ 0.07  \\
    \hline
    Inception (avg)           & 5.59$\pm$0.38   & 5.44$\pm$0.08   & 5.16$\pm$0.37  & - \\
    Internal classifier (avg) & 3.64$\pm$0.47   & 3.70$\pm$0.27   & 4.04$\pm$0.37  & -  \\
    \bottomrule
  \end{tabular}}
  \caption{Inception Score \cite{SalimansARXIV2016} for different GAN training
    methods.  Since the score depends on the classifier, we used  code from \cite{SalimansARXIV2016} as well as our own
    small convnet CIFAR-10 classifier for evaluation (achieves 83\%
    accuracy). All scores are computed using 10,000 samples. The top pair are 
    scores on the final models. GANs are known to be unstable, and results are 
    sometimes cherry-picked. So, the bottom pair are scores averaged across models sampled from 
    different iterations of training after it stopped improving.
    }
  \label{tab:inception}
\end{table}

\begin{figure*}[t]
  \centering
  \begin{tabular}{ccc}
    \includegraphics[width=0.3\textwidth]{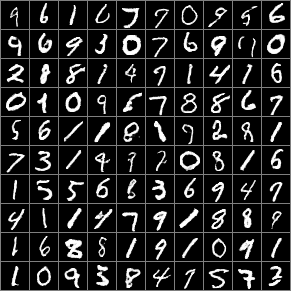} &
    \includegraphics[width=0.3\textwidth]{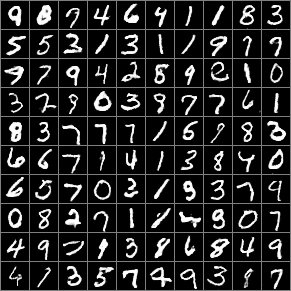} &
    \includegraphics[width=0.3\textwidth]{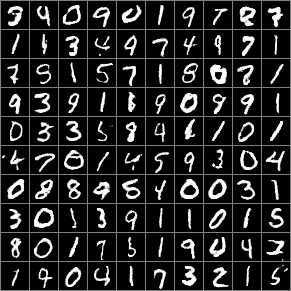}\\
    \includegraphics[width=0.3\textwidth]{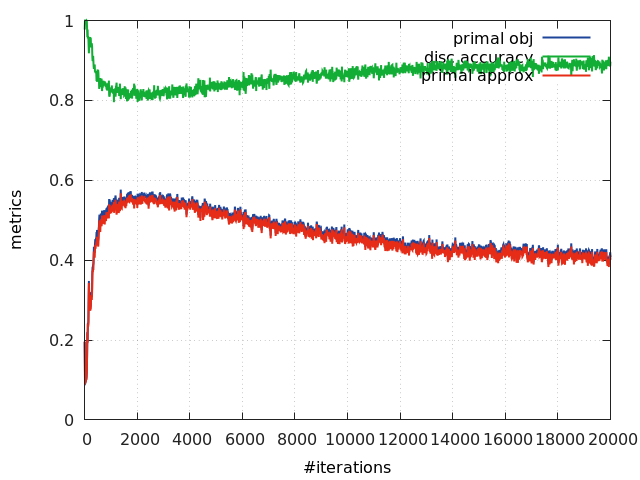} &
    \includegraphics[width=0.3\textwidth]{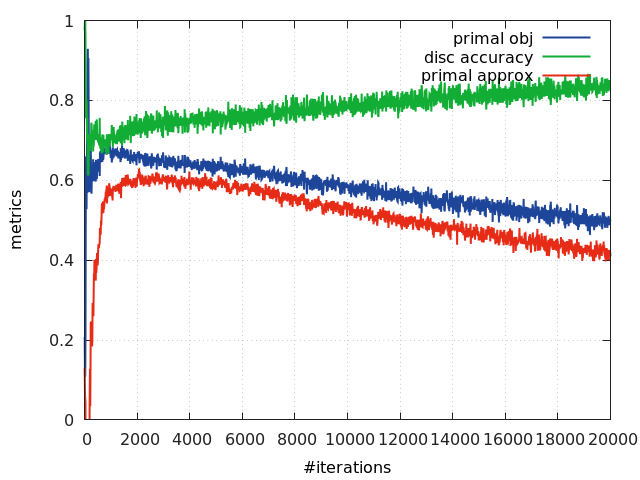} &
    \includegraphics[width=0.3\textwidth]{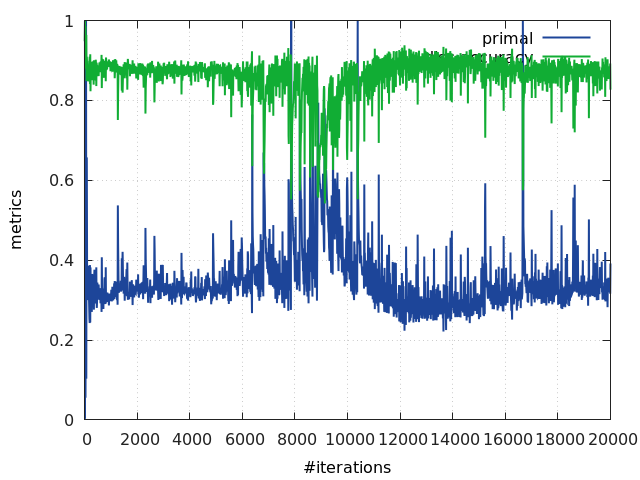}\\
    Score Linearization & Cost Linearization & GAN\\
  \end{tabular}
  \caption{Nonlinear discriminator experiments on MNIST, and their training
  curves, showing the primal objective, the approximation, and the discriminator
accuracy.
  Here we are showing typical runs with the compared methods (not
cherry-picked).}
  \label{fig:nonlin-mnist}
\end{figure*}

\begin{figure*}[t]
  \centering
  \begin{tabular}{ccc}
    \includegraphics[width=0.3\textwidth]{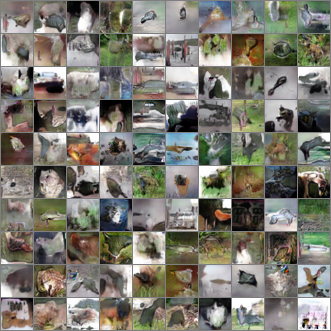} &
    \includegraphics[width=0.3\textwidth]{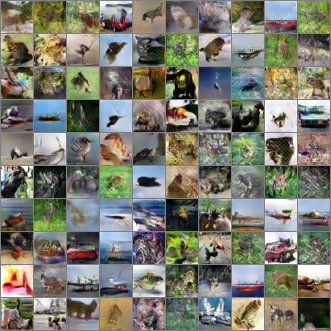} &
    \includegraphics[width=0.3\textwidth]{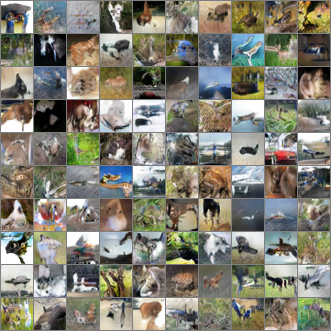}\\
    Score Linearization & Cost Linearization & GAN\\
  \end{tabular}
  \caption{Nonlinear discriminator experiments on CIFAR-10, learning curves and 
  samples organized by class are provided in the supplementary material.}
  \label{fig:nonlin-cifar}
\end{figure*}
\vspace{-0.2cm}
\section{Related Work}\label{sec_relate_work}
\vspace{-0.1cm}
A thorough review of the research devoted to generative modeling is beyond the scope
of this paper.  In this section we focus on GANs \cite{GoodfellowARXIV2014} and
review the most related work that has not been discussed throughout the paper. 


Our dual formulation reveals a close connection to moment-matching 
objectives widely seen in many other models.  MMD \cite{GrettonJMLR2012} is one
such related objective, and has been used in deep generative models in
\cite{LiARXIV2015, dziugaite2015training}.  \cite{SalimansARXIV2016}
proposed a range of techniques to improve GAN training, including the usage of
feature matching.  Similar techniques are also common in style transfer
\cite{gatys2015neural}.  In addition to these, moment-matching objectives are very
common for exponential family models \cite{wainwright2008graphical}.  Common to all these works is the use of 
fixed moments.  The Wasserstein objective proposed for GAN training in
\cite{ArjovskyARXIV2017} can also be thought of as a form of moment matching, where the features are
part of the discriminator and they are adaptive.  The main difference between
 our dual GAN with linear discriminators and other forms of adaptive moment
matching is that we adapt the weighting of features by optimizing non-parametric dual
parameters, while other works mostly adopt a parametric model to adapt features.



Duality has also been studied to understand and improve GAN training.
\cite{NowozinNIPS2016} pioneered  work that uses duality to derive new GAN
training objectives from other divergences.  \cite{ArjovskyARXIV2017}
also used duality to derive a practical objective for training GANs from other
distance metrics.  Compared to previous work, instead of coming up with new
objectives, we instead used duality on the original GAN objective and aim to better
optimize the discriminator.



Beyond what has already been discussed, there has been a range of other
techniques developed to improve or extend GAN training, \eg, 
\cite{ImARXIV2016, HuangARXIV2016, ZhangARXIV2016, ChenARXIV2016, ZhaoICLR2017,
LondonNIPSWS2016} just to name a few.


\vspace{-0.3cm}
\section{Conclusion}
\label{sec_conclusion}
\vspace{-0.2cm}
To conclude, we introduced `Dualing GANs,' a framework which considers duality
based formulations for the duel between the discriminator and the generator.
Using the dual formulation provides opportunities to train the discriminator
better.  This helps remove the instability in training for linear
discriminators, and we also adapted this framework to non-linear discriminators.
The dual formulation also provides  connections to other techniques.
In particular, we discussed a close link to  moment matching
techniques, and showed that the cost function linearization for non-linear
discriminators recovers the original gradient direction in standard GANs.
We hope that our results spur further research in this direction to obtain a
better understanding of the GAN objective and its intricacies.

\bibliography{LinGAN}
\bibliographystyle{plain}
\clearpage

\appendix
\section{Minibatch objective}

Standard GAN training are motivated from a maximin formulation on an expectation
objective
\begin{equation}
  \max_\theta \min_\bw \expt_{p_d(\bx)}[\log D_\bw(\bx)] + \expt_{p_z(\bz)}[\log(1
  - D_\bw(G_\theta(\bz)))]
\end{equation}
where $p_d$ is the true data distribution, and $p_z$ is the prior distribution
on $\bz$.

In practice, however, a minibatch of data $\bX = \{\bx_1, ..., \bx_n\}$ and
noise $\bZ=\{\bz_1, ..., \bz_n\}$ are sampled each time, and one gradient update
is made to update $\bw$ and $\theta$ each.

In our formulation, in particular the dual GAN with linear discriminators, we
can solve the inner optimization problem over $\bw$ on minibatch samples $\bX$
and $\bZ$ to optimality, $\theta$ is then updated with the optimal $\bw$.  This
effectively makes the optimization problem take the following form
\begin{equation}
  \max_\theta \min_\bw \expt_{\bX, \bZ} [f(\bw, \theta, \bX, \bZ)], \quad
\text{where}\quad f(\bw, \theta, \bX, \bZ) = \frac{1}{n}\sum_i \hat{f}_x(\bw, \theta, \bx_i)
+ \frac{1}{n}\sum_i \hat{f}_z(\bw, \theta, \bz_i)
\end{equation}
where $\hat{f}_x$ and $\hat{f}_z$ are the individual loss functions.  Using this
notation, the original GAN problem can be represented as
\begin{equation}
  \max_\theta \min_\bw \expt_{\bx, \bz}[f(\bw, \theta, \{\bx\}, \{\bz\})] =
  \max_\theta \min_\bw \expt_{\bX, \bZ}[f(\bw, \theta, \bX, \bZ)]
\end{equation}
since, $\bx_i$ and $\bz_i$ are drawn i.i.d. from corresponding distributions.

Let $\bw^* = \argmin_\bw \expt_{\bX, \bZ} [f(\bw, \theta,
\bX, \bZ)]$, we have
\begin{equation}
  \expt_{\bX, \bZ} [\min_\bw f(\bw, \theta, \bX, \bZ)] \le \expt_{\bX, \bZ}[
  f(\bw^*, \theta, \bX, \bZ)] = \min_\bw \expt_{\bX, \bZ}[f(\bw, \theta, \bX,
  \bZ)],
\end{equation}
which means our minibatch algorithm is actually optimizing a lower bound on the
theoretical GAN objective, this introduces a bias that decreases with minibatch
size, but guarantees that the optimization is still valid.

On the other hand, interleaving minibatch training with partial optimization of
$\bw$ (not all the way to optimality) makes the standard GAN training behave
differently, however the exact properties of this process is hard to
characterize and beyond the scope of this paper.

\section{Proof of Claim 1}

\begin{claim}
The dual program to the minimization task
\begin{equation*}
\min_\bw \quad \frac{C}{2} \|\bw\|^2_2 + \frac{1}{2n}\sum_i \log(1 +
\exp(-\bw^\top\bx_i)) + \frac{1}{2n}\sum_i \log (1 + \exp(\bw^\top
G_\theta(\bz_i))).
\end{equation*}
reads as follows:
\begin{align}
  \max_\lambda \quad& g(\theta, \lambda) = -\frac{1}{2C}\left\|\sum_i \lambda_{\bx_i} \bx_i - \sum_i \lambda_{\bz_i}
  G_\theta(\bz_i)\right\|^2 + \frac{1}{2n}\sum_i H(2n\lambda_{\bx_i}) +
  \frac{1}{2n}\sum_i H(2n\lambda_{\bz_i}), \nonumber\\
  \suchthat \quad& \forall i, \quad 0 \le \lambda_{\bx_i} \le \frac{1}{2n}, \quad 0 \le
  \lambda_{\bz_i} \le \frac{1}{2n}.
\end{align}
with binary entropy $H(u) = - u\log u-(1-u)\log(1-u)$, and the optimal solution
to the original problem $\bw^*$ can be expressed with optimal
$\lambda_{\bx_i}^*$ and $\lambda_{\bz_i}^*$ as
$$
\bw^* = \frac{1}{C}\left(\sum_i \lambda_{\bx_i}^* \bx_i - \sum_i
  \lambda_{\bz_i}^* G_\theta(\bz_i)\right) 
$$
\end{claim}

\begin{proof}
We introduce auxillary variables $\xi_{\bx_i} = \bw^\top\bx_i$ and
$\xi_{\bz_i}=-\bw^\top G_\theta(\bz_i)$, the original minimization problem can
then be transformed into the following equality constrained problem
\begin{eqnarray}
  & \min_\bw & \frac{C}{2} \|\bw\|^2_2 + \frac{1}{2n}\sum_i \log(1 +
  e^{-\xi_{\bx_i}}) + \frac{1}{2n}\sum_i \log (1 + e^{-\xi_{\bz_i}})) \\
  & \suchthat & \forall i, \quad \xi_{\bx_i} = \bw^\top\bx_i, \quad \xi_{\bz_i}
  = -\bw^\top G_\theta(\bz_i). \nonumber
\end{eqnarray}
The corresponding Lagrangian has the following form
\begin{eqnarray}
  L(\bw, \xi, \lambda, \theta) &=& \frac{C}{2} \|\bw\|^2_2 + \frac{1}{2n}\sum_i \log(1 +
  e^{-\xi_{\bx_i}}) + \frac{1}{2n}\sum_i \log (1 + e^{-\xi_{\bz_i}})) \nonumber
  \\
  & & + \sum_i \lambda_{\bx_i}(\xi_{\bx_i} - \bw^\top\bx_i) + \sum_i
  \lambda_{\bz_i}(\xi_{\bz_i} + \bw^\top G_\theta(\bz_i))
\end{eqnarray}
Set the derivatives with respect to the primal variables to 0, we get
\begin{eqnarray}
  \pdiff{L}{\bw} &=& C\bw - \sum_i \lambda_{\bx_i} \bx_i + \sum_i
  \lambda_{\bz_i} G_\theta(\bz_i) = 0 \\
  \label{eq:linear-gan-xi-x}
  \pdiff{L}{\xi_{\bx_i}} &=& -\frac{1}{2n} \frac{e^{-\xi_{\bx_i}}}{1 +
e^{-\xi_{\bx_i}}} + \lambda_{\bx_i} = 0 \\
  \label{eq:linear-gan-xi-z}
\pdiff{L}{\xi_{\bz_i}} &=& -\frac{1}{2n} \frac{e^{-\xi_{\bz_i}}}{1 +
e^{-\xi_{\bz_i}}} + \lambda_{\bz_i} = 0.
\end{eqnarray}
We can then represent the primal variables using the $\lambda$'s,
\begin{eqnarray}
  \label{eq:linear-gan-w}
  \bw &=& \frac{1}{C}\left(\sum_i \lambda_{\bx_i} \bx_i - \sum_i \lambda_{\bz_i}
G_\theta(\bz_i)\right) \\
  \xi_{\bx_i} &=& \log \frac{1-2n\lambda_{\bx_i}}{2n\lambda_{\bx_i}} \\
  \xi_{\bz_i} &=& \log \frac{1-2n\lambda_{\bz_i}}{2n\lambda_{\bz_i}}.
\end{eqnarray}
Eq.(\ref{eq:linear-gan-xi-x}) and (\ref{eq:linear-gan-xi-z}) also introduced
extra constraints on $\lambda_{\bx_i}$ and $\lambda_{\bz_i}$, as follows
\begin{equation}
  \forall i, \quad 0 \le \lambda_{\bx_i} \le \frac{1}{2n}, \quad 0 \le
  \lambda_{\bz_i} \le \frac{1}{2n}.
\end{equation}
Substituting the primal variables back to the Lagrangian, we get the dual
objective
\begin{eqnarray}
  g(\theta, \lambda) &=& \frac{C}{2}\left\|\frac{1}{C}\left(\sum_i \lambda_{\bx_i} - \sum_i
\lambda_{\bz_i}\right)\right\|_2^2 - \frac{1}{2n}\log (1 - 2n\lambda_{\bx_i}) -
  \frac{1}{2n}\log(1 - 2n\lambda_{\bz_i}) \nonumber\\
  && + \sum_i \lambda_{\bx_i} \log\frac{1-2n\lambda_{\bx_i}}{2n\lambda_{\bx_i}}
  + \sum_i \lambda_{\bz_i} \log\frac{1-2n\lambda_{\bz_i}}{2n\lambda_{\bz_i}} +
  \frac{1}{C}\left\|\sum_i \lambda_{\bx_i}\bx_i -
  \lambda_{\bz_i}G_\theta(\bz_i)\right\|_2^2 \nonumber\\
  &=& -\frac{1}{2C}\left\|\sum_i \lambda_{\bx_i} \bx_i - \sum_i \lambda_{\bz_i}
  G_\theta(\bz_i)\right\|^2 + \frac{1}{2n}\sum_i H(2n\lambda_{\bx_i}) +
  \frac{1}{2n}\sum_i H(2n\lambda_{\bz_i})
\end{eqnarray}
The overall dual problem is therefore
\begin{eqnarray}
  & \max_\lambda & g(\theta, \lambda) = -\frac{1}{2C}\left\|\sum_i \lambda_{\bx_i} \bx_i - \sum_i \lambda_{\bz_i}
  G_\theta(\bz_i)\right\|^2 + \frac{1}{2n}\sum_i H(2n\lambda_{\bx_i}) +
  \frac{1}{2n}\sum_i H(2n\lambda_{\bz_i}), \nonumber\\
  & \suchthat & \forall i, \quad 0 \le \lambda_{\bx_i} \le \frac{1}{2n}, \quad 0 \le
  \lambda_{\bz_i} \le \frac{1}{2n}.
\end{eqnarray}
Once we have solved for the optimal $\lambda^*$, we can recover the optimal
primal solution $\bw^*$ using \eqref{eq:linear-gan-w}.
\end{proof}


\section{Setting the step size $\Delta_k$ in the trust-region method}

Pursuing this trust-region intuition, we can alternatively choose $\Delta_k$
based on the accuracy of the model $m_{k,\theta}(\bs)$. To this end it is often
convenient to introduce the acceptance ratio
\be
\rho = \frac{f(\bw_k,\theta)-f(\bw_k+\bs,\theta)}{f(\bw_k,\theta) - m_{k,\theta}(\bs)},
\label{eq:AccRatio}
\ee
which compares the real function value difference to the modeled one. If the
acceptance ratio $\rho$ deviates significantly from 1 on either side, we may opt
to decrease the trust region $\Delta_k$ and resolve the program given in
Eq. (4) of the main paper, instead of accepting the step.

Intuitively, if $\rho$ specified in \equref{eq:AccRatio} is far from 1, the model function does not fit well the original objective. To obtain a better fit we resolve the program using a smaller trust region size $\Delta_k$.

\section{Proof of Claim 2}

\begin{claim}
\label{clm:dualModel}
The dual program to $\min_\bs m_{k,\theta}(\bs)$ s.t.~$\frac{1}{2}\|\bs\|_2^2
\leq \Delta_k$ with model function given as 

\begin{align*}
  m_{k, \theta}(\bs) =\,& \frac{C}{2}\|\bw_k + \bs\|_2^2 + \frac{1}{2n}\sum_i
  \log\left(1 + \exp\left(-F(\bw_k, \bx_i) - \bs^\top\pdiff{F(\bw_k,
  \bx_i)}{\bw}\right)\right) \\
  & + \frac{1}{2n}\sum_i \log \left(1 + \exp\left(F(\bw_k, G_\theta(\bz_i)) +
  \bs^\top\pdiff{F(\bw_k, G_\theta(\bz_i))}{\bw}\right)\right)
\end{align*}

is the following:
\begin{eqnarray*}
  \max_{\lambda}
  &&\frac{C}{2}\|\bw_k\|_2^2 -
  \frac{1}{2(C+\lambda_T)}\left\|-C\bw_k + \sum_i \lambda_{\bx_i}\pdiff{F(\bw_k,
  \bx_i)}{\bw} - \sum_i \lambda_{\bz_i}\pdiff{F(\bw_k,
  G_\theta(\bz_i))}{\bw}\right\|_2^2 \\
  &&+ \frac{1}{2n}\sum_i H(2n\lambda_{\bx_i}) + \frac{1}{2n}\sum_i
H(2n\lambda_{\bz_i}) - \sum_i \lambda_{\bx_i} F_{\bx_i} + \sum_i \lambda_{\bz_i}
F_{\bz_i} -\lambda_T \Delta_k \\
\suchthat && \lambda_T \geq 0 \quad\quad \forall i,\quad 0\leq\lambda_{\bx_i}\leq
\frac{1}{2n}, \quad 0\leq \lambda_{\bz_i}\leq \frac{1}{2n}.
\end{eqnarray*}
The optimal $\bs^*$ to the original problem can
be expressed through optimal $\lambda_T^*, \lambda_{\bx_i}^*, \lambda_{\bz_i}^*$ as
\begin{equation*}
  \bs^*= \frac{1}{C + \lambda_T^*}\left(\sum_i \lambda_{\bx_i}^*\pdiff{F(\bw_k,
      \bx_i)}{\bw}
    - \sum_i \lambda_{\bz_i}^*\pdiff{F(\bw_k, \bz_i)}{\bw}\right) - \frac{C}{C +
    \lambda_T^*}\bw_k
\end{equation*}
\end{claim}

\begin{proof}  In this optimization problem, the free variable is $\bs$.  We
  introduce short hand notations $F_{\bx_i} = F(\bw_k, \bx_i),
  F_{\bz_i}=F(\bw_k, G_\theta(\bz_i)), \nabla F_{\bx_i} = \pdiff{F(\bw_k, \bx_i)}{\bw}$
  and $\nabla F_{\bz_i} = \pdiff{F(\bw_k, G_\theta(\bz_i))}{\bw}$.  With these
  extra notations we can simplify the primal problem as

  \begin{equation}
    m_{k, \theta}(\bs) = \frac{C}{2}\|\bw_k + \bs\|_2^2 + \frac{1}{2n}\sum_i \log\left(1 +
  e^{-F_{\bx_i}-\bs^\top\nabla F_{\bx_i}}\right) + \frac{1}{2n}\sum_i
    \log\left(1 + e^{F_{\bz_i} + \bs^\top\nabla F_{\bz_i}}\right)
  \end{equation}

  Again, we introduce auxillary variables $\xi_{\bx_i} = \bs^\top\nabla
  F_{\bx_i}$ and
  $\xi_{\bz_i}=-\bs^\top \nabla F_{\bz_i}$, and obtain the following constrained
  optimization problem
  \begin{align}
    \min_{\bs, \xi} \quad& \frac{C}{2}\|\bw_k + \bs\|_2^2 + \frac{1}{2n}\sum_i \log\left(1 +
    e^{-F_{\bx_i}- \xi_{\bx_i}}\right) + \frac{1}{2n}\sum_i
    \log\left(1 + e^{F_{\bz_i} - \xi_{\bz_i} }\right) \\
    \suchthat \quad& \xi_{\bx_i} = \bs^\top\nabla F_{\bx_i}, \quad
    \xi_{\bz_i} = -\bs^\top\nabla F_{\bz_i}, \quad \forall i \nonumber \\
    & \frac{1}{2}\|\bs\|^2 \le \Delta_k \nonumber
  \end{align}
  
  The corresponding Lagrangian is the following
  \begin{align}
    L(\bw, \xi, \lambda) =&\; \frac{C}{2}\|\bw_k + \bs\|_2^2 + \frac{1}{2n}\sum_i \log\left(1 +
  e^{-F_{\bx_i} - \xi_{\bx_i}}\right) + \frac{1}{2n}\sum_i
  \log\left(1 + e^{F_{\bz_i} - \xi_{\bz_i}}\right) \nonumber\\
  &\; + \sum_{i} \lambda_{\bx_i}(\xi_{\bx_i} - \bs^\top\nabla F_{\bx_i}) +
  \sum_{i} \lambda_{\bz_i}(\xi_{\bz_i} + \bs^\top \nabla F_{\bz_i}) +
  \lambda_T \left(\frac{1}{2}\|\bs\|^2 - \Delta_k\right)
  \end{align}

  Setting the derivatives of the primal variables with respect to the Lagrangian
  to 0, we get
  \begin{eqnarray}
    \pdiff{L}{\bs} &=& C(\bw_k + \bs) - \sum_i\lambda_{\bx_i}\nabla F_{\bx_i}
    + \sum_i \lambda_{\bz_i}\nabla F_{\bz_i} + \lambda_T \bs = 0 \\
    \pdiff{L}{\xi_{\bx_i}} &=& -\frac{1}{2n}\frac{e^{-F_{\bx_i}-\xi_{\bx_i}}}{1 +
  e^{-F_{\bx_i}-\xi_{\bx_i}}} + \lambda_{\bx_i} = 0 \\
  \pdiff{L}{\xi_{\bz_i}} &=& -\frac{1}{2n}\frac{e^{F_{\bz_i}-\xi_{\bz_i}}}{1 +
e^{F_{\bz_i} - \xi_{\bz_i}}} + \lambda_{\bz_i} = 0
  \end{eqnarray}
  Therefore
  \begin{eqnarray}
    \bs &=& \frac{1}{C + \lambda_T}\left(\sum_i \lambda_{\bx_i}\nabla F_{\bx_i}
  - \sum_i \lambda_{\bz_i}\nabla F_{\bz_i}\right) - \frac{C}{C + \lambda_T}\bw_k \\
  \xi_{\bx_i} &=& \log\frac{1-2n\lambda_{\bx_i}}{2n\lambda_{\bx_i}} - F_{\bx_i}
  \\
  \xi_{\bz_i} &=& \log\frac{1-2n\lambda_{\bz_i}}{2n\lambda_{\bz_i}} + F_{\bz_i},
  \end{eqnarray}
  which includes the equation for $\bs^*$.

  Next we substitute these back to the Lagrangian to obtain the dual objective.
  We introduce another short hand notation $\square=\sum_i \lambda_{\bx_i}
  \nabla F_{\bx_i} - \sum_i \lambda_{\bz_i} \nabla F_{\bz_i}$, then
  $\bs=\frac{1}{C+\lambda_T}\square - \frac{C}{C+\lambda_T}\bw_k$, and the dual
  objective can be written as
  \begin{eqnarray}
    g(\lambda) &=& \frac{C}{2}\left\|\frac{1}{C+\lambda_T}\left(\lambda_T \bw_k +
  \square\right)\right\|^2 -\frac{1}{2n}\sum_i \log(1 - 2n\lambda_{\bx_i}) - \frac{1}{2n}\sum_i \log(1 -
2n\lambda_{\bz_i}) \nonumber\\
&& +\sum_i \lambda_{\bx_i}\left(\log\frac{1-2n\lambda_{\bx_i}}{2n\lambda_{\bx_i}} -
F_{\bx_i} - \frac{1}{C+\lambda_T}\left(\square -
C\bw_k\right)^\top\nabla
F_{\bx_i}\right) \nonumber \\
&& + \sum_i \lambda_{\bz_i}\left(\log\frac{1-2n\lambda_{\bz_i}}{2n\lambda_{\bz_i}} +
F_{\bz_i} + \frac{1}{C+\lambda_T}\left(\square -
C\bw_k\right)^\top\nabla
F_{\bz_i}\right) \nonumber \\
&& + \frac{\lambda_T}{2}\left\|\frac{1}{C+\lambda_T}\left(\square -
C\bw_k\right)\right\|^2 - \lambda_T \Delta_k \nonumber \\
&=& \frac{1}{2n}\sum_i H(2n\lambda_{\bx_i}) + \frac{1}{2n}\sum_i
H(2n\lambda_{\bz_i}) - \sum_i \lambda_{\bx_i} F_{\bx_i} + \sum_i \lambda_{\bz_i}
F_{\bz_i} -\lambda_T \Delta_k \nonumber\\
&& + \frac{C}{2(C+\lambda_T)^2}\|\lambda_T \bw_k + \square\|^2 -
\frac{1}{C+\lambda_T}(\square - C\bw_k)^\top\square +
\frac{\lambda_T}{2(C+\lambda_T)^2}\|\square - C\bw_k\|^2 \nonumber \\
&=& \frac{1}{2n}\sum_i H(2n\lambda_{\bx_i}) + \frac{1}{2n}\sum_i
H(2n\lambda_{\bz_i}) - \sum_i \lambda_{\bx_i} F_{\bx_i} + \sum_i \lambda_{\bz_i}
F_{\bz_i} -\lambda_T \Delta_k \nonumber\\
&&-\frac{1}{2(C+\lambda_T)}\|\square - C\bw_k\|^2 + \frac{C}{2}\|\bw_k\|^2
  \end{eqnarray}
  which is exactly the dual objective in the claim.
\end{proof}

\section{More Experiment Details}

\subsection{Toy dataset}

The toy 2D dataset used in the paper consists of a mixture of 5 2D Gaussian
components, the Gaussians have covariance matrix of $0.1I$ with means being
uniformly spaced on a circle of radius 2.

Here we present additional results on an extra 8-mode dataset, where
each of the 8
components in the 8-mode dataset is a Gaussian distribution with a covariance
matrix of $0.02I$, and again the means of the components are arranged on a circle of
radius 2.
We use both
datasets to investigate properties such as low probability regions and low
separation of modes.

To train the linear GAN we employ RBF features based on a set of anchor points
$\{x_1, ..., x_n\}$, then for an arbitrary $x$, the features for $x$ is computed
as the following
$$
\phi(x) = \left[\frac{\exp(-\frac{1}{T}\|x-x_1\|^2)}{Z}, ...,
\frac{\exp(-\frac{1}{T}\|x-x_n\|^2)}{Z}\right]^\top, \quad \text{where} \quad Z=\sum_i
\exp(-\frac{1}{T}\|x-x_i\|^2).
$$
We set $T$ to 0.2 for all experiments.

Experiment results are shown in \figref{fig:ToyDataSampleResults}.




\begin{figure*}[t]
\centering
\begin{tabular}{c||cc}
&Ours & Standard GAN \\\hline\hline
\rotatebox{90}{8-mode data}&
\includegraphics[trim={0cm 0cm 0cm 0cm},clip,width=0.3\textwidth]{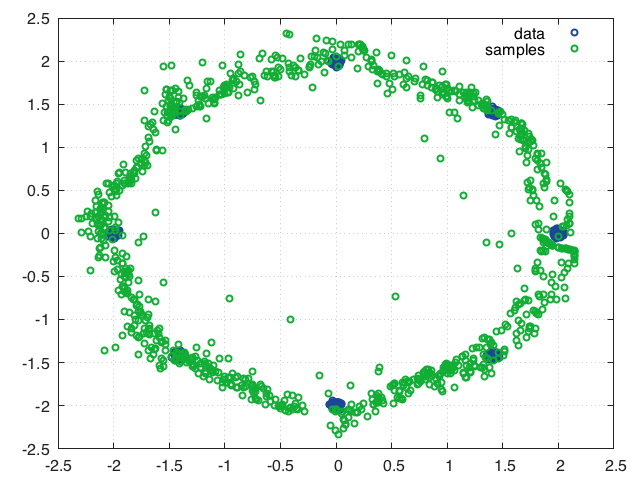}&
\includegraphics[trim={0cm 0cm 0cm
0cm},clip,width=0.3\textwidth]{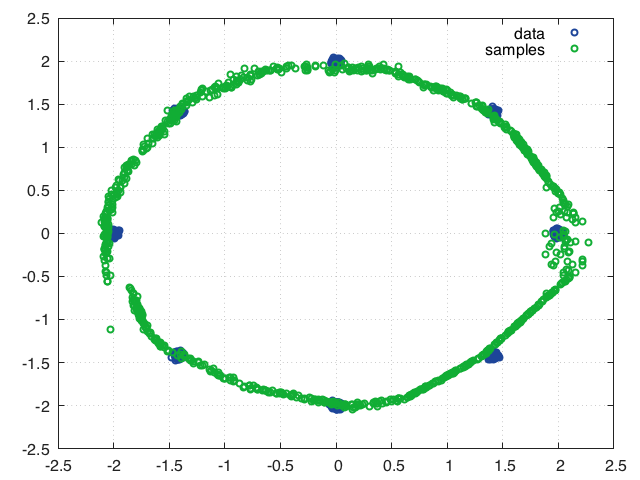}\\\hline
\rotatebox{90}{5-mode data}&
\includegraphics[trim={0cm 0cm 0cm 0cm},clip,width=0.3\textwidth]{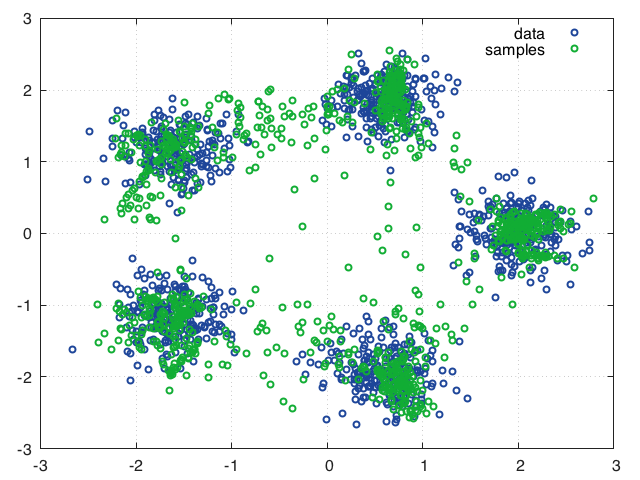}&
\includegraphics[trim={0cm 0cm 0cm 0cm},clip,width=0.3\textwidth]{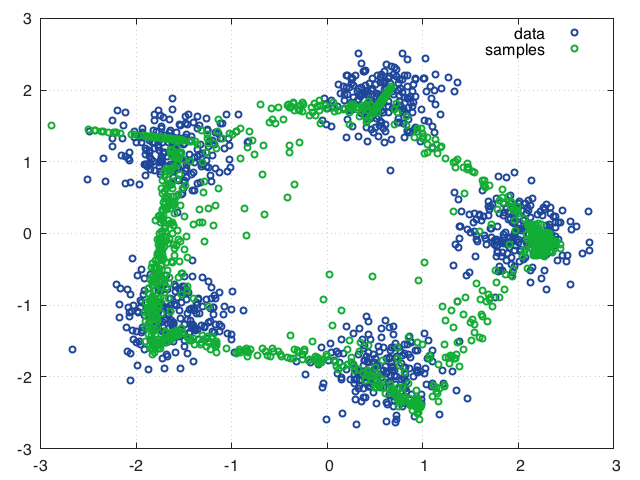}
\end{tabular}
\caption{Original data (blue) and samples obtained from the learned generator
(green) for our approach (left) and standard GAN (right). We show results for
8-mode (top) and 5-mode data (bottom). For each of the approaches we demonstrate
usage of RBF features.}
\label{fig:ToyDataSampleResults}
\end{figure*}

\subsection{MNIST Model Details}

We used a generator architecture similar to that in [17].
Instead of directly projecting the initial hidden variables to $4\times 4$ images, we
first feed it through a fully-connected hidden layer. In the intermediate layers,
we use $4\times 4$ upconvolution kernels with stride 2, ReLU activation, and batch normalization.
At the output layer, we fed it through 1 extra $3\times 3$ convolution layer without changing 
the image size. Instead of Tanh output activation, we use a Sigmoid function, and our 
data takes pixel value between 0 and 1. For all of our experiments, our initial 
hidden dimension is 32. For discriminator, our pretrained MNIST convnet uses 3
3x3 convolution layers with max pooling and ReLU activation, and 2 fully connected
hidden layers with ReLU activation as well. 

\subsection{CIFAR-10}

For the generator, we use an architecture similar
to the one described for MNIST.  
For the discriminator, it is similar to MNIST as well, except before each max pooling
operation there are 2 convolutional and ReLU layers instead of 1. The width of the 
network here is also greater than the one for MNIST experiment.



We provide in
\figref{fig:nonlin-cifar-supp} the primal objective, the discriminator accuracy
and the model function value (primal approx.) throughout training, and top samples
for each class.  From the training curves we see that the proposed trust-region
cost-linearization technique is significantly more stable than either the score
linearization or standard GANs.  The score linearization method does a better
job approximating the discriminator at the begining, but then suffered from bad
solution to the dual problem given by the Ipopt solver.

\begin{figure*}[t]
  \centering
  \begin{tabular}{ccc}
    \includegraphics[width=0.25\textwidth]{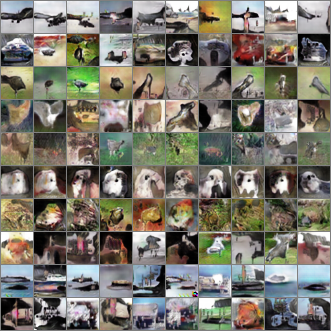} &
    \includegraphics[width=0.25\textwidth]{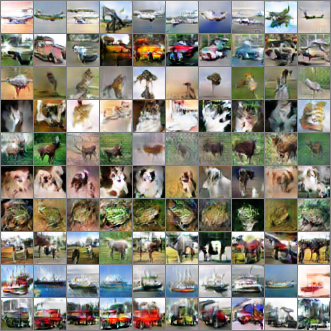} &
    \includegraphics[width=0.25\textwidth]{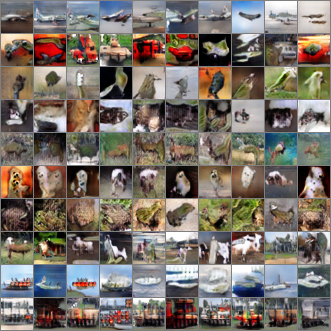}\\

    \includegraphics[width=0.3\textwidth]{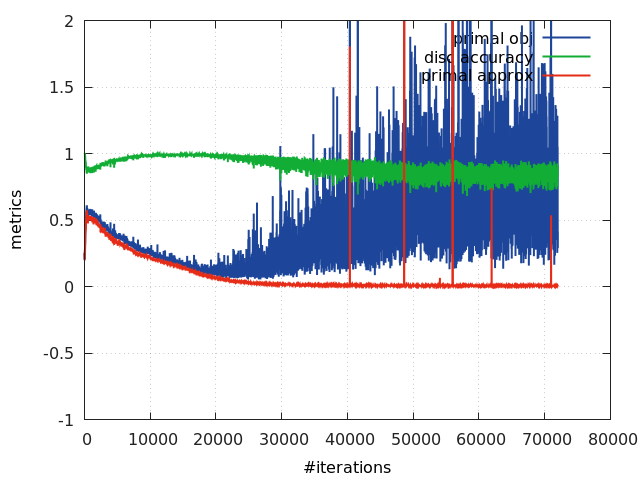} &
    \includegraphics[width=0.3\textwidth]{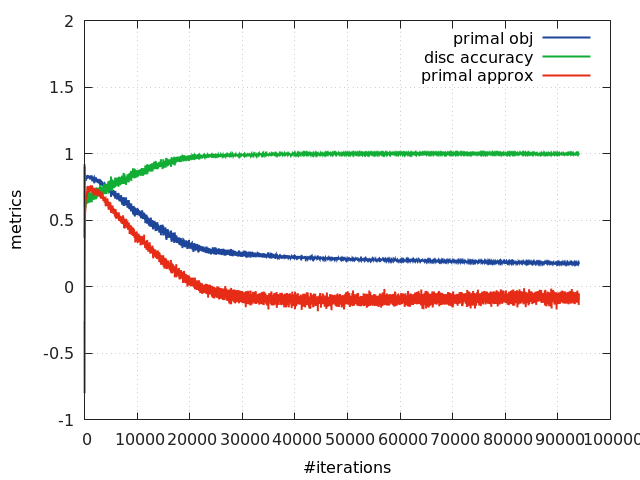} &
    \includegraphics[width=0.3\textwidth]{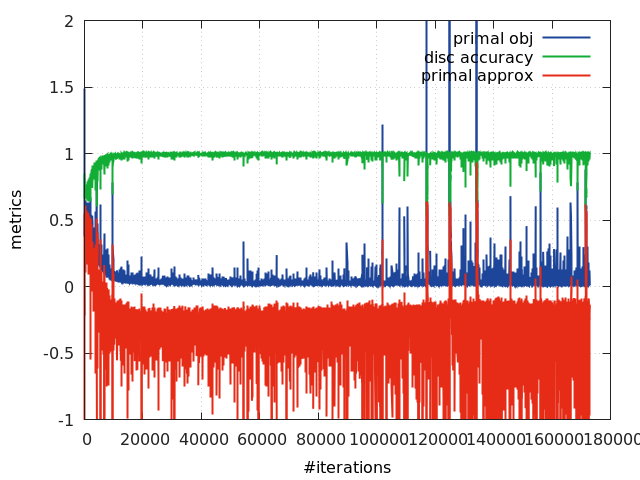}\\
    Score Linearization & Cost Linearization & GAN\\
  \end{tabular}
  \caption{Nonlinear discriminator experiments on CIFAR-10, learning curves and 
  samples organized by class.}
  \label{fig:nonlin-cifar-supp}
\end{figure*}

\end{document}